\documentclass[10pt,twocolumn,letterpaper]{article}

\usepackage{iccv} 
\usepackage{times}
\usepackage{graphicx}
\usepackage{amsmath}
\usepackage{amssymb}
\usepackage{amsthm} 
\usepackage{fixmath,fixltx2e,xspace}
\usepackage{bbm,booktabs}

\newcommand{\Lap}{{L}}
\newcommand{\namealgo}{EIG-SE(3)\xspace}
\newcommand{\govse}{Govindu-SE(3)\xspace}

\providecommand{\abs}[1]{\lvert#1\rvert} 
\DeclareMathOperator{\med}  {med}
\DeclareMathOperator{\trace}{tr}
\DeclareMathOperator{\rank} {rank}
\DeclareMathOperator{\diag} {diag}
\DeclareMathOperator{\nullsp}  {null}
\newcommand{\vzero} {\mathbf{0}}  
\newcommand{\vone} 	{\mathbbm{1}}

\newcommand{\quarta}{ [ 0 \thickspace 0  \thickspace 0 \thickspace 1]}

\newcommand{\ra}[1]{\renewcommand{\arraystretch}{#1}}

\theoremstyle{plain}
\newtheorem{proposition}{Proposition}

\usepackage[pagebackref=true,breaklinks=true,letterpaper=true,colorlinks,bookmarks=false]{hyperref}

\iccvfinalcopy 

\def\iccvPaperID{1343} 

\begin{document}

\title{Spectral Motion Synchronization in SE(3)}

\author{Federica Arrigoni, Andrea Fusiello\\
DIEGM - University of Udine\\
Via delle Scienze, 208 - Udine (Italy)
\and
Beatrice Rossi\\
AST Lab - STMicroelectronics\\
Via Olivetti, 2 - Agrate Brianza (Italy)
}

\maketitle

\begin{abstract}
This paper addresses the problem of motion synchronization (or
averaging) and describes a simple, closed-form solution based on a
spectral decomposition, which does not consider rotation and
translation separately but works straight in $SE(3)$, the manifold of
rigid motions.
Besides its theoretical interest, being the first closed form solution
in $SE(3)$, experimental results show that it compares favourably with
the state of the art both in terms of precision and speed.

\end{abstract}

\section{Introduction}

In this paper we address the \emph{motion synchronization} (a.k.a.~\emph{motion
averaging} or \emph{motion registration}) problem in the Special Euclidean Group, $SE(3)$, which consists in
recovering $n$ \emph{absolute} motions, i.e. rigid 3D displacements expressed 
in an absolute (external) coordinate system, starting from a redundant set of 
\emph{relative} (pairwise) motions.  This problem appears in the context of 
\emph{structure-from-motion} (SfM) --   where the absolute motions represent
orientations and positions of cameras capturing a 3D scene, and \emph{multiple point-set registration} -- that requires to find the rigid transformations that bring multiple 3D point-sets into alignment.

In the literature on multiple point-set registration,
the origins of motion synchronization can be traced back to the \emph{frame space} methods \cite{Sharp02} that optimize the internal coherence of the network of rotations and translations applied to the local coordinate frames, as opposed to solutions that optimize a cost function depending on the distance of corresponding points (e.g. \cite{Pen96,BenShm98}).

In the structure-from-motion literature, \emph{global} methods, that first solve for the motion by optimizing the network of relative motions and leave the 3D structure recovery at the end, are fairly recent (e.g.  \cite{MartinecP07}), although the origins of these approaches can be traced back to \cite{Govindu01}. 

Almost all the techniques address the motion synchronization problem by breaking it up into rotation and translation, and solving the two synchronization problems separately.

For what regards \emph{rotation synchronization}, a theoretical analysis of the problem is reported  in \cite{averaging}. The absolute rotations can be recovered by using the quaternion representation of $SO(3)$, as done in \cite{Govindu01}, or by distributing the error over cycles in the graph of neighbouring views \cite{Sharp02}. 
In \cite{Hartley11} a cost based on the $\ell_1$-norm is used to average relative rotations, where each absolute rotation is updated in turn using the Weiszfeld algorithm.  Martinec in \cite{MartinecP07} casts the problem as the optimization of an objective function based on the $\ell_2$-norm of the compatibility error between relative estimates and unknown absolute orientations, and this approach is extended in \cite{Arie12} where approximate solutions are computed either via spectral decomposition or semidefinite programming.
The sum of unsquared deviations is proposed in \cite{Wang2013} as a more robust self consistency error.
Chatterjee in \cite{Chatterjee_2013_ICCV} exploits the Lie-group structure of  rotations, and combines an $\ell_1$ averaging in the tangent space with an iteratively reweighted least squares (IRLS) approach. In \cite{RGodec} the rotation synchronization problem is reformulated in terms of low-rank and sparse matrix decomposition.

As for \emph{translation synchronization} methods, a discriminating factor relevant to our analysis is whether they use only relative motion information, or, in addition, exploit image point correspondences (e.g.~\cite{MartinecP07,Snavely14}). Let us focus on the former, which are more similar to our $SE(3)$ approach.
In  \cite{Govindu01} absolute positions are initialized as the least squares solution of a linear system of equations in the pairwise directions and orientations, and they are then improved through IRLS. In \cite{Brand04}  a fast spectral solution to translation synchronization is proposed by reformulating the problem in terms of graph embedding. The method presented in \cite{Ozyesil2013stable} first computes pairwise directions through a robust subspace estimation and then derives absolute translations using a semidefinite relaxation. In \cite{Jiang13} a linear solution which minimizes a geometric error in camera triplets is presented, while in \cite{moulon13} relative translations are computed through an a-contrario trifocal tensor estimation, and then absolute positions are recovered by using an $\ell_\infty$ formulation.

A different approach for motion synchronization is followed in \cite{Govindu04} where rotations and translations are \emph{jointly} considered.
This method exploits the Lie-group structure of $SE(3)$ and uses an iterative scheme in which at each step the absolute motions are approximated by averaging relative motions in tangent space. In \cite{Govindu06} robustness is introduced through random sampling in the measurement graph. 
Originally proposed in the SfM framework, this technique was also applied to
multiple point-set registration \cite{Govindu14} and simultaneous localization and mapping (SLAM) \cite{Agrawal06}, where   additional information about data reliability is introduced in the form of covariance matrices.

In this paper we propose a novel method for synchronizing relative motions in $SE(3)$, 
based on  computing the least four eigenvectors of a $4n \times 4n$ matrix.
Our  approach is the first one that works in $SE(3)$ \emph{and} has a closed-form solution, being based on a spectral decomposition. 
It can be seen as the extension to $SE(3)$ of the spectral synchronization proposed in \cite{Singer2011} for $SO(2)$ and generalized in \cite{Singer2011b,Arie12} to $SO(3)$.

The simple matrix formulation of our method leads immediately to a weighted formulation, in much the same way as \cite{Wang2013} did for $SO(N)$, that allows to  embed it  into an IRLS scheme in order to handle rogue measurements.

Experimental results on synthetic and real data show that it compares favourably with
the state of the art both in terms of accuracy and efficiency.


\section{Our Method}

The \emph{motion synchronization} problem consists in recovering $n$ \emph{absolute} motions, i.e. rigid displacements in $\mathbb{R}^3$ expressed in an absolute (external) coordinate system, starting from a redundant set of \emph{relative} (pairwise) motions. Such relative information is usually corrupted by a diffuse noise, in addition to sparse gross errors (outliers).
Let $\mathcal{E} \subseteq \{1,2,\dots,n\} \times \{1,2,\dots,n\}$ denote the set of available pairs, which can be viewed as the set of edges of an undirected finite simple graph $\mathcal{G} = (\mathcal{V}, \mathcal{E})$, where vertices in $\mathcal{V}$ correspond to absolute motions. 
In practical applications this graph is far from complete, due to the lack of overlap between some pairs of images/scans.
However, there is a significant level of redundancy among relative motions in general datasets, which can be used to distribute the error over all the nodes, avoiding drift in the solution.

Each motion can be viewed as an element of the Special Euclidean Group $SE(3)$, which is the semi-direct product of the Special Orthogonal Group $SO(3)$ with $\mathbb{R}^3$. As a matrix group, $SE(3)$ is a subgroup of the General Linear Group GL(4), thus the inverse of a displacement and composition of displacements reduce to matrix operations. 	
Accordingly, each absolute motion is described by a homogeneous transformation
\begin{equation}
M_i =
\begin{pmatrix}
R_i & \mathbf{t_i} \\
\vzero & 1
\end{pmatrix} \in SE(3)
\end{equation}
where $R_i \in SO(3)$ and $\mathbf{t_i} \in \mathbb{R}^3$ represent the rotation and translation components of the i-th transformation. Similarly, each relative motion can be expressed as
\begin{equation}
M_{ij} =
\begin{pmatrix}
R_{ij} & \mathbf{t_{ij}} \\
\vzero & 1
\end{pmatrix} \in SE(3)
\end{equation}
where $R_{ij} {\in} SO(3)$ and $\mathbf{t_{ij}} {\in} \mathbb{R}^3$ encode the
transformation between frames $i$ and $j$. The link between absolute and
relative motions is encoded by the \emph{compatibility constraint}
\begin{equation}
M_{ij} = M_i M_j^{-1}
\label{eq_compatibility}
\end{equation}
which is equivalent to $ R_{ij} = R_i R_j^{\mathsf{T}} $ and $ \mathbf{t_{ij}} = - R_i R_j^{\mathsf{T}} \mathbf{t_j} + \mathbf{t_i}$
by considering separately the rotation and translation terms.
Relative motions can be seen as measurements for the ratios of the unknown group
elements. Finding group elements from noisy measurements of their ratios is also
known as the \emph{synchronization} problem \cite{Giridhar06,Singer2011}.


The remainder of this section is organized as follows.  In Sec.~\ref{sec_exact} we describe
properties that hold when all the relative information is exact, necessary to
define our technique. Then we derive our spectral
solution to motion synchronization (Sec.~\ref{sec_noise}).  
In Sec.~\ref{sec_outliers} our method
is embedded into an IRLS framework in order to handle outliers among relative
motions.
Finally, Sec.~\ref{sec_SEN} briefly presents the extension of our method in $SE(N)$.

\subsection{The Exact Case}
\label{sec_exact}

The absolute transformations can be recovered from \eqref{eq_compatibility} --
up to a global motion -- if we express it in a useful equivalent way that takes
into account all the relative information at once. For simplicity of exposition,
we first consider the case where all the pairwise motions are available.

Let $X \in \mathbb{R}^{4n \times 4n}$ denote the block-matrix containing the
ideal (noise free) relative motions and let $M \in \mathbb{R}^{4n \times 4}$ be
the stack of the absolute motions, namely
\begin{equation}
M=
\begin{bmatrix}
M_1 \\
M_2 \\
\dots \\
M_n
\end{bmatrix},
\quad
X =
\begin{pmatrix}
I_4 & M_{12} & \dots & M_{1n} \\
M_{21} & I_4 & \dots & M_{2n} \\
\dots &  &  & \dots \\
M_{n1} & M_{n2} & \dots & I_4 \\
\end{pmatrix}
\end{equation}
where $I_4$ indicates the $4 \times 4$ identity matrix. If $ M^{-\flat}  \in \mathbb{R}^{4 \times 4n}$ is the concatenation of the inverse of absolute motions, i.e. 
$ M^{-\flat} =
\begin{bmatrix}
M_1^{-1} & M_2^{-1} & \dots & M_n^{-1}
\end{bmatrix}$, then the compatibility constraint turns into
$
X = M  M^{-\flat} 
$,
and hence $\rank(X) = 4$. 
Note that here $X$ is not symmetric positive semidefinite, in contrast to the case of $SO(3)$.
Since $ M^{-\flat}  M = n I_4$, we obtain
\begin{equation}
XM = n M
\label{eq_eig_full}
\end{equation}
which means that -- in the absence of noise -- the columns of $M$ are 4 (independent) eigenvectors of $X$
associated to the eigenvalue $n$. 
Equation \eqref{eq_eig_full} is equivalent to 
\begin{equation}
(n I_{4n}  - X)M  = 0.
\label{eq_svd_full}
\end{equation}
Thus the columns of $M$ are a basis for the $4$-dimensional null-space of $\Lap = (n I_{4n}  - X)$. The matrix $\Lap$ resembles a block Laplacian, as it will we clarified ahead.

Conversely, any basis $U$ for $\nullsp(\Lap )$ will not coincide with
$M$ in general, since it will not be composed of euclidean
motions. Specifically, it will not coincide with $\quarta$
in every fourth row. 
In order to recover $M$ from $U$ it is sufficient to choose a different basis for $\nullsp(\Lap )$ that satisfies such constraint, which can be found by taking a suitable linear combination of the columns of $U$.
More precisely, let $P \in \mathbb{R}^{n \times 4n}$ be the 0-1 matrix  such that $PU \in \mathbb{R}^{n \times 4}$ consists of the rows of $U$ with indices multiple of four. 
The coefficient $\mathbold{\alpha}, \mathbold{\beta} \in \mathbb{R}^4$ of the linear combination are solution of 
\begin{equation}
\begin{gathered}
P U \mathbold{\alpha} = \mathbf{0},  \quad 
P U \mathbold{\beta} = \mathbf{1}
\end{gathered}
\label{eq_ls}
\end{equation}
where the first equation has a three-dimensional solution space. Let
$\mathbold{\alpha_1}, \mathbold{\alpha_2}, \mathbold{\alpha_3}$ be a basis for
the null-space of $PU$. Thus the columns of $M$ corresponding to rotations coincide (up to a permutation) with 
$[U \mathbold{\alpha_1},U \mathbold{\alpha_2}, U \mathbold{\alpha_3}]$
and $M$ is recovered as $M =  U [ \mathbold{\alpha_1},
  \mathbold{\alpha_2},  \mathbold{\alpha_3}, \mathbold{\beta}]$.

Note that this post-processing on the eigenvectors is not required for the spectral method in $SO(3)$, since any orthogonal basis for the null-space of $L$ coincides (up to a permutation) with the stack of the absolute rotations.


We now consider the case of missing data, in which the graph $\mathcal{G}$ is not complete. In this situation missing pairwise motions correspond to zero blocks in $X$. Let $A \in \mathbb{R}^{n \times n}$ be the adjacency matrix of $\mathcal{G}$ and 
let $D \in \mathbb{R}^{n \times n}$ be the degree matrix of $\mathcal{G}$, i.e.
the diagonal matrix that contains the degree of node $i$ in its entry $D_{i,i}$.
It can be seen that Eq.~\eqref{eq_svd_full} generalizes to
\begin{equation}
((D-A)\otimes \vone_{4 \times 4}) \circ X) M =0
\label{eq_svd_missing}
\end{equation}
where $\otimes$ denotes the Kronecker product and $\circ$ denotes the Hadamard product. The matrix $(D-A)$ is the Laplacian matrix of the graph $\mathcal{G}$, which gets ``inflated" to a $4 \times 4$-block structure by the Kronecker product with  
$\vone_{4 \times 4}$ (a matrix filled by ones), and then is multiplied entry-wise with $X$.
Thus the columns of $M$ are a basis for the $4$-dimensional null-space of $\Lap = ((D-A)\otimes \vone_{4 \times 4})\circ X$.

If $\mathcal{G}$ is complete then $D=(n-1)I_n$ and $A= \vone_{n \times n} -I_n$, hence the matrix $\Lap$ reduces to the previous one.

\subsubsection*{Weighted graph.}
In some applications we are given non-negative weights $w_{ij}$ that reflect the
reliability of the pairwise measurements. In other words, $\mathcal{G}$ is a weighted graph with real weights, stored in the 
the symmetric adjacency matrix $A = [w_{ij}] $.  Accordingly, the degree matrix $D$ of the weighted graph is defined as $D_{i,i} = \sum_{ j \text{ s.t. } (i,j) \in \mathcal{E}} w_{ij}$.
Equation \eqref{eq_svd_missing} still holds with these definitions, thus our spectral method extends to \emph{weighted} motion synchronization.


\subsection{Dealing with Noise}
\label{sec_noise}

We now consider the case where the pairwise motions are corrupted by noise, hence they do not satisfy equations \eqref{eq_compatibility} and \eqref{eq_svd_missing} exactly. 
Thus the goal is to recover the absolute motions such that they are ``maximally compatible'' with the available relative information.
In order to address this motion synchronization problem, we consider an algebraic cost function that measures the residuals (in the Frobenius norm sense) of Equation \eqref{eq_svd_missing}, namely
\begin{equation}
\min_{M \in SE(3)^n} \left\| \widehat{\Lap}M \right\|_F^2
\label{eq_min}
\end{equation}
with the additional constraint $\left\| \mathbf{m_4} \right\|_F = c $ in order to fix the global scale. Here $\mathbf{m_4}$ denotes the fourth column of $M$,
$\widehat{X}$ denotes a noisy version of the ideal matrix $X$, which contains the measured relative motions $\widehat{M}_{ij} \in SE(3)$, and 
$\widehat{\Lap} = ((D-A)\otimes \vone_{4 \times 4})\circ \widehat{X}$.
Hereafter we will consistently use the hat accent to denote noisy measurements.
Such a problem is difficult to solve since the feasible set $SE(3)^n = SE(3) \times \dots \times SE(3)$ is non-convex.

In order to make the computation tractable, we do not solve Problem \eqref{eq_min} directly, but we proceed as follows. 
First, we look for an orthogonal basis for the (approximated) 4-dimensional null-space of $\widehat{\Lap}$, by solving 
the following optimization problem
\begin{equation}
\min_{U^{\mathsf{T}}U = n I_4} \left\| \widehat{\Lap}U \right\|_F^2.
\label{eq_relaxed}
\end{equation}
In other words, we solve the homogeneous system of equations $\widehat{\Lap}U = \mathbf{0}$ in the least-squares sense, where the solution space is known to have approximately dimension $4$.

Then, we find an estimate for $M$ within this space by forcing the solution to coincide with $\quarta$ in every fourth row. Finally, we project in $SO(3)$ all the $3 \times 3$ blocks corresponding to rotations by using Singular Value Decomposition (SVD).

\begin{proposition}
\label{prop_eig}
Problem \eqref{eq_relaxed} admits a closed-form solution, which is given by the 4 eigenvectors of $\widehat{\Lap}^{\mathsf{T}} \widehat{\Lap}$ associated to the 4 smallest eigenvalues.
\end{proposition}

\begin{proof}
We first observe that Problem \eqref{eq_relaxed} coincides with
\begin{equation}
\min_{U^{\mathsf{T}}U = n I_4} \trace(U^{\mathsf{T}} (\widehat{\Lap}^{\mathsf{T}} \widehat{\Lap}) U).
\label{eq_trace}
\end{equation}
Let $\mathcal{F}$ be the unconstrained cost function corresponding to this problem, namely
\begin{equation}
\mathcal{F} (U) = \trace(U^{\mathsf{T}} (\widehat{\Lap}^{\mathsf{T}} \widehat{\Lap}) U) + \trace(\Lambda(U^{\mathsf{T}}U - n I_4 ))
\end{equation}
where $\Lambda \in \mathbb{R}^{4 \times 4}$ is a symmetric matrix of unknown Lagrange multipliers.
Setting to zero the partial derivatives of $\mathcal{F}$ with respect to $U$ we obtain
\begin{equation}
\frac{ \partial \mathcal{F} }{ \partial U } = 2  (\widehat{\Lap}^{\mathsf{T}} \widehat{\Lap}) U + 2 U \Lambda = 0 \Rightarrow  
(\widehat{\Lap}^{\mathsf{T}} \widehat{\Lap}) U = - U \Lambda.
\label{eq_derivative}
\end{equation}
Let $\mathbf{u_i}$ be any four eigenvectors of $\widehat{\Lap}^{\mathsf{T}} \widehat{\Lap}$ 
(normalized so that $\left\| \mathbf{u_i} \right\| = \sqrt{n}$) and let $\lambda_i$ be the corresponding eigenvalues.
Then $U = [\mathbf{u_1} | \mathbf{u_2} | \mathbf{u_3} | \mathbf{u_4} ]$ satisfies both \eqref{eq_derivative} and the constraint 
$U^{\mathsf{T}}U = n I_4$, with $\Lambda = - \diag(\lambda_1, \lambda_2, \lambda_3, \lambda_4)$ (indeed $\widehat{\Lap}^{\mathsf{T}} \widehat{\Lap}$ admits an orthonormal basis of real eigenvectors since it is symmetric). In other words, any quadruple of eigenvectors is a stationary point for the objective function $\mathcal{F}$. The minimum is attained in \eqref{eq_trace} if $\mathbf{u_i}$ are the 4 least eigenvectors of $\widehat{\Lap}^{\mathsf{T}} \widehat{\Lap}$.
\end{proof}

Proposition \ref{prop_eig} guarantees that the solution to problem \eqref{eq_relaxed} is given by the 4 least eigenvectors of $\widehat{\Lap}^{\mathsf{T}} \widehat{\Lap}$, which coincide with the 4 least right singular vectors in the Singular Value Decomposition (SVD) of $\widehat{\Lap}$.
Such a solution represents the best (in the Frobenius norm sense) 4-dimensional approximation for $\nullsp(\widehat{\Lap})$.
Within such a space, we find the solution that is closest to have
every fourth row equal to $\quarta$  by solving system \eqref{eq_ls} in the least-squares sense. Then, such a solution is projected onto $SE(3)^n$ -- as in \cite{belta2002euclidean} -- by forcing  every fourth row  to $\quarta$ and projecting $3 \times 3$ rotation  blocks onto $SO(3)$ through 
SVD.


This technique has the advantage of being extremely fast, as motion synchronization is cast to eigenvalue decomposition of a $4n \times 4n$ matrix.
Moreover, in practical application the measurement graph $\mathcal{G}$ is sparse, thus employing sparse eigen-solvers (such as {\sc Matlab} \texttt{eigs}) increases its efficiency. 
From the computational complexity point of view, the Lanczos method  (implemented by \texttt{eigs})  is ``nearly linear"
since every iteration is linear
in $n$, if the matrix is sparse, but the number of iterations is not
constant.


\subsection{Dealing with Outliers}
\label{sec_outliers}

The fact that our spectral method copes easily with weights on individual
relative motions allows a straightforward extension to gain resilience to
rogue input measures via Iteratively Reweighted Least Squares (IRLS).

First, we solve \eqref{eq_relaxed} to obtain an estimate for $M$ with given
weights\footnote{The initial weights are all 1 by default, but they can be
  initialized from any reliability information coming from the relative motion
  estimation procedure.} as explained in the previous section, then we update the weights using the current
estimate of absolute motions, and these steps are iterated until convergence.
In our experiments we used the Cauchy weight function \cite{Holland77}
\begin{equation}
w_{ij} = \genfrac{}{}{}{}{1}{ \raisebox{-5pt}{$1+\left(\frac{r_{ij}}{c}\right)^2$}}
\end{equation}  
where 
$r_{ij} = \| \widehat{M}_{ij} - M_i M_j^{-1} \|_F$.  
The tuning constant $c$ have been chosen, as customary, based on the median absolute deviation (MAD):  $c = 1.482 \; \theta
\med (\abs{\mathbf{r} - \med(\mathbf{r}) }) $, where $\med()$ is the median operator, $\mathbf{r}$ is the vectorization of the residuals $r_{ij}$, and  $\theta=2$.

\subsection{Generalization to SE(N)}
\label{sec_SEN}

In this paper we focus on $SE(3)$ because this group arises in several applications.
However, it is straightforward to see that our analysis and the derived spectral method apply equally well to any dimension. 

Suppose that we are given a redundant number of pairwise ratios $M_{ij} \in SE(N)$, and we want to estimate the associated group elements $M_i\in SE(N)$, which represent rigid displacements in $\mathbb{R}^N$.
If the graph is complete then -- in the absence of noise -- the block-matrix 
$X \in \mathbb{R}^{(N+1)n \times (N+1)n}$ has rank $N+1$, and the columns of $M$ are $N+1$ eigenvectors of $X$ with eigenvalue $n$.
If the graph is not complete then Equation \eqref{eq_svd_missing} still hold, and hence the columns of $M$
form a basis for the $(N+1)$-dimensional null-space of $L$.
Thus we can generalize our spectral method to synchronize elements of $SE(N)$, by computing the $N+1$ least eigenvectors of 
$\widehat{\Lap}^{\mathsf{T}} \widehat{\Lap}$.

\section{Experiments}


In this section we evaluate our spectral method -- henceforth called \namealgo \ -- on both simulated and real data in terms of accuracy, execution cost and robustness to outliers. 
We compare \namealgo to several techniques from the state-of-the-art.
All the experiments are performed in {\sc Matlab} on a MacBook Air with i5 dual-core @ 1.3 GHz.
In order to compare estimated and ground-truth absolute motions, we find the optimal transformation that aligns them by applying single averaging \cite{Hartley11} for the rotation term and least-squares for the scale and translation. We use the angular distance and Euclidean norm to measure the accuracy of absolute rotations and translations respectively.

\subsection{Simulated Data}



\begin{figure*}[!htbp]
\centering
\includegraphics[width=0.33\linewidth]{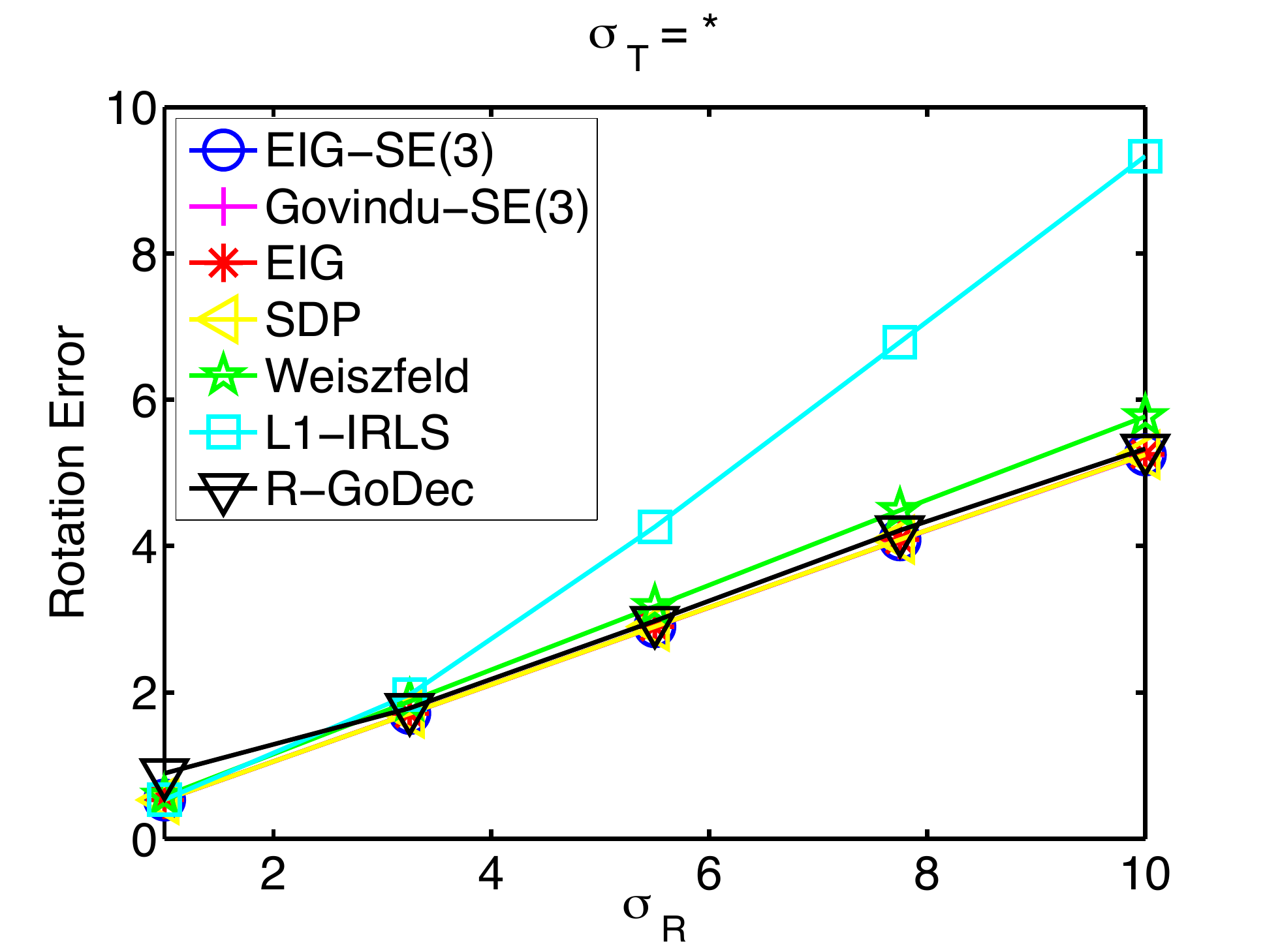} \quad
\includegraphics[width=0.33\linewidth]{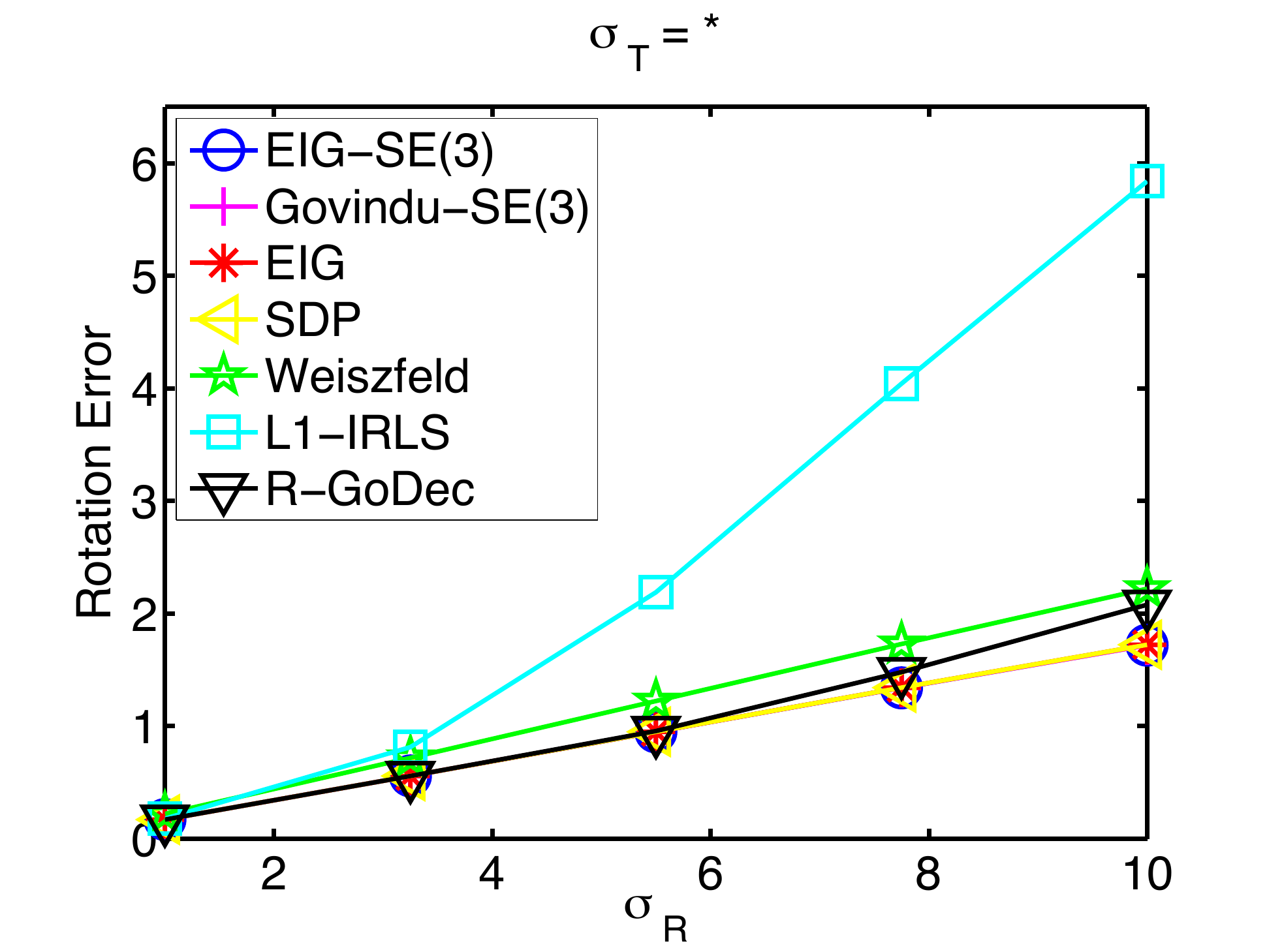}  
\caption{Mean angular errors (degrees) on the absolute rotations with $d=5$ (left) and $d=30$ (right). The value of $\sigma_T$ is meaningless.}
\label{exp_noiseR}
\end{figure*}

\begin{figure*}[!htbp]
\centering
\includegraphics[width=0.33\linewidth]{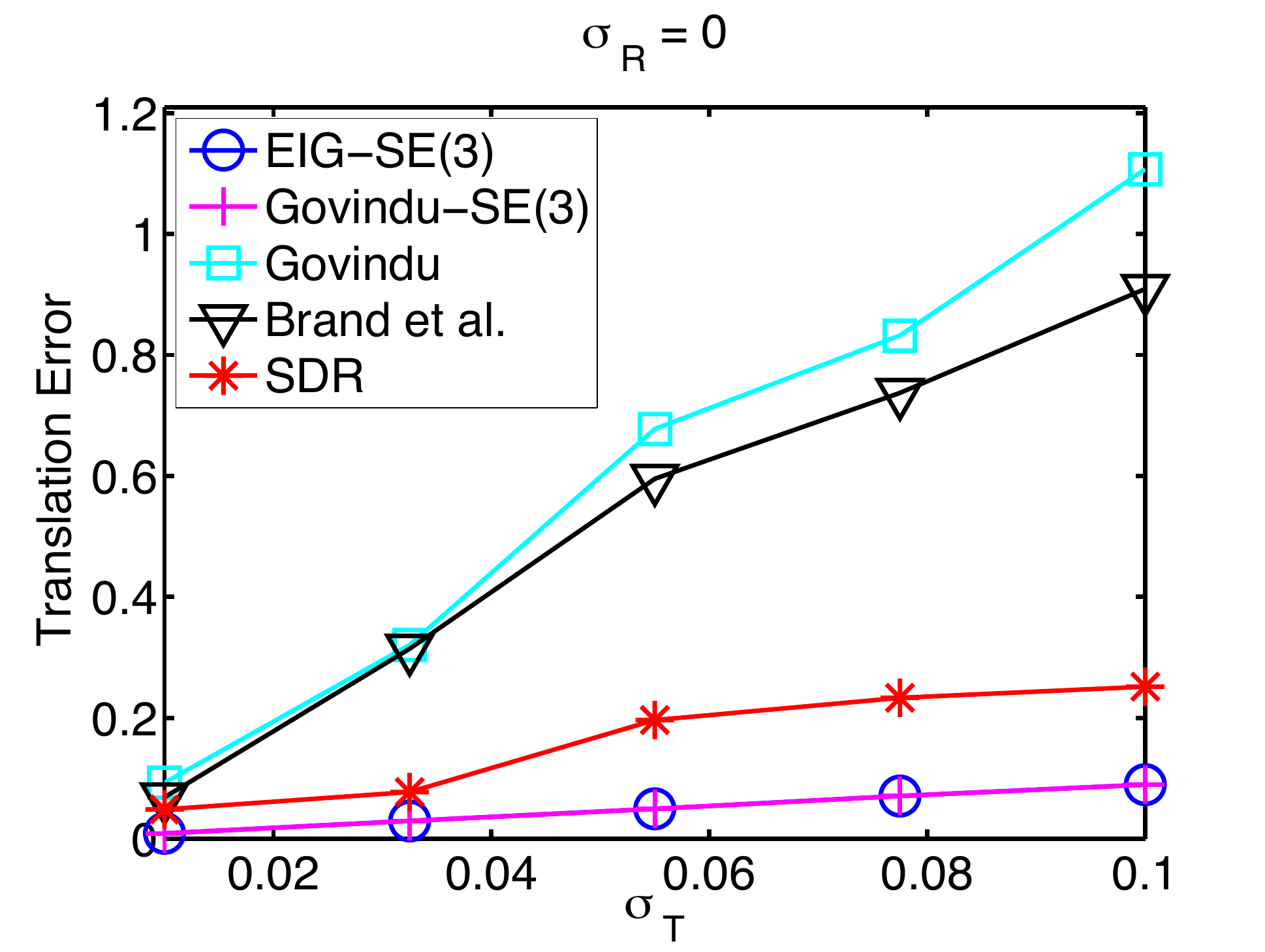} \quad
\includegraphics[width=0.33\linewidth]{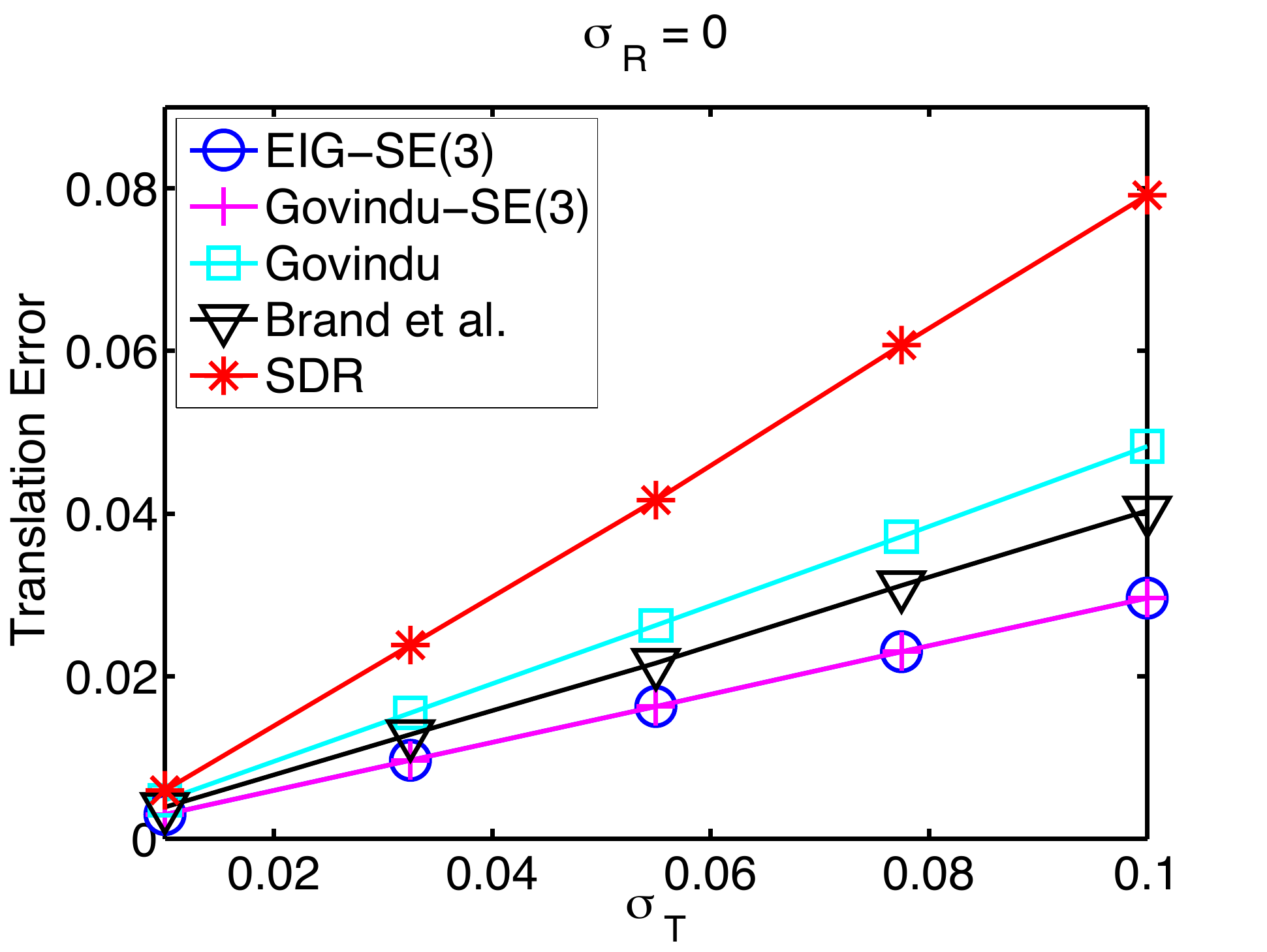}  
\caption{Mean errors on the absolute translations with $d=5$ (left) and $d=30$ (right).}
\label{exp_noiseT}
\end{figure*}

In these experiments  we consider $n$ absolute motions in which rotations are sampled from random Euler angles and translation components follow a standard Gaussian distribution.
The level of sparsity of the measurement graph is defined through the average degree $d$ of nodes. The available pairwise motions are corrupted by a multiplicative noise, where the rotation component has axis uniformly distributed over the unit sphere and angle following a Gaussian distribution with zero mean and standard deviation $\sigma_R \in [1^{\circ},10^{\circ}]$, and  the translation components are sampled from a Gaussian distribution with zero mean and standard deviation 
$\sigma_T \in [0.01,0.1]$.  
In this way we perturb both direction and magnitude of pairwise translations.
All the results are averaged over $50$ trials.


We evaluate the effect of noise on rotations and translations both separately and together, by considering $n=100$ absolute motions, in the cases $d=5$ and $d=30$, which correspond to about $95\%$ and $70\%$ of missing pairs, respectively. Higher values of $d$ correspond to better conditioned problems, with the same qualitative behaviour as $d=30$. Please note that in the real cases reported in Tab.~\ref{tab_large}, the percentage of missing pairs ranges from 30\% to 90\%.


\paragraph{Rotation synchronization.}

As for rotations, besides \govse  \cite{Govindu04}, we consider general synchronization techniques such as the Weiszfeld algorithm \cite{Hartley11}, spectral relaxation \cite{Arie12} (EIG), semidefinite programming \cite{Arie12} (SDP),
the L1-IRLS algorithm \cite{Chatterjee_2013_ICCV}, and the \textsc{R-GoDec} algorithm \cite{RGodec}.
Methods based on quaternions (such as \cite{Govindu01}) have been already proved inferior to the other methods in \cite{MartinecP07}.
The code of L1-IRLS is available on-line, while in the other cases we used our implementation.

Figure \ref{exp_noiseR} reports the mean angular errors on the absolute rotations as a function of $\sigma_R$, obtained by running the rotation synchronization techniques mentioned above.
The best accuracy is obtained by \namealgo together with EIG, SDP and \govse .
On the contrary, the robust approaches \textsc{R-GoDec}, L1-IRLS  and Weiszfeld yield worse results, to different extents, because they inherently trade robustness for statistical efficiency.

The noise on relative translations does not have any influence on absolute rotations, hence the value of $\sigma_T$ is meaningless in this experiment.

\paragraph{Translation synchronization.}

As for translations, we consider only methods working in frame space, i.e. not requiring  point  correspondences, such as  SDR \cite{Ozyesil2013stable}, the graph-embedding approach by Brand et al.~\cite{Brand04} and the works of Govindu \cite{Govindu01,Govindu04}. 
Among these methods, only \namealgo and \govse  \cite{Govindu04}  are influenced by the noise on the translation norms, for they work in $SE(3)$, while this does not influence the remaining algorithms, which take as input relative translation directions.
The code of SDR is available on-line, while in the other cases we use our implementation.
In this simulation we do not perturb the relative rotations ($\sigma_R=0$), thus all the methods are given ground-truth relative/absolute rotations.
Noise on rotational component influences also the translation errors with results qualitatively similar to those reported here.

Figure \ref{exp_noiseT} shows the mean errors on the absolute translations as
a function of $\sigma_T$ (units are commensurate with the simulated data),
obtained by running the techniques mentioned above. 
Both \namealgo and \govse  outperform all the analysed methods in terms of accuracy.

\begin{figure*}[!htbp]
\centering
\includegraphics[width=0.33\linewidth]{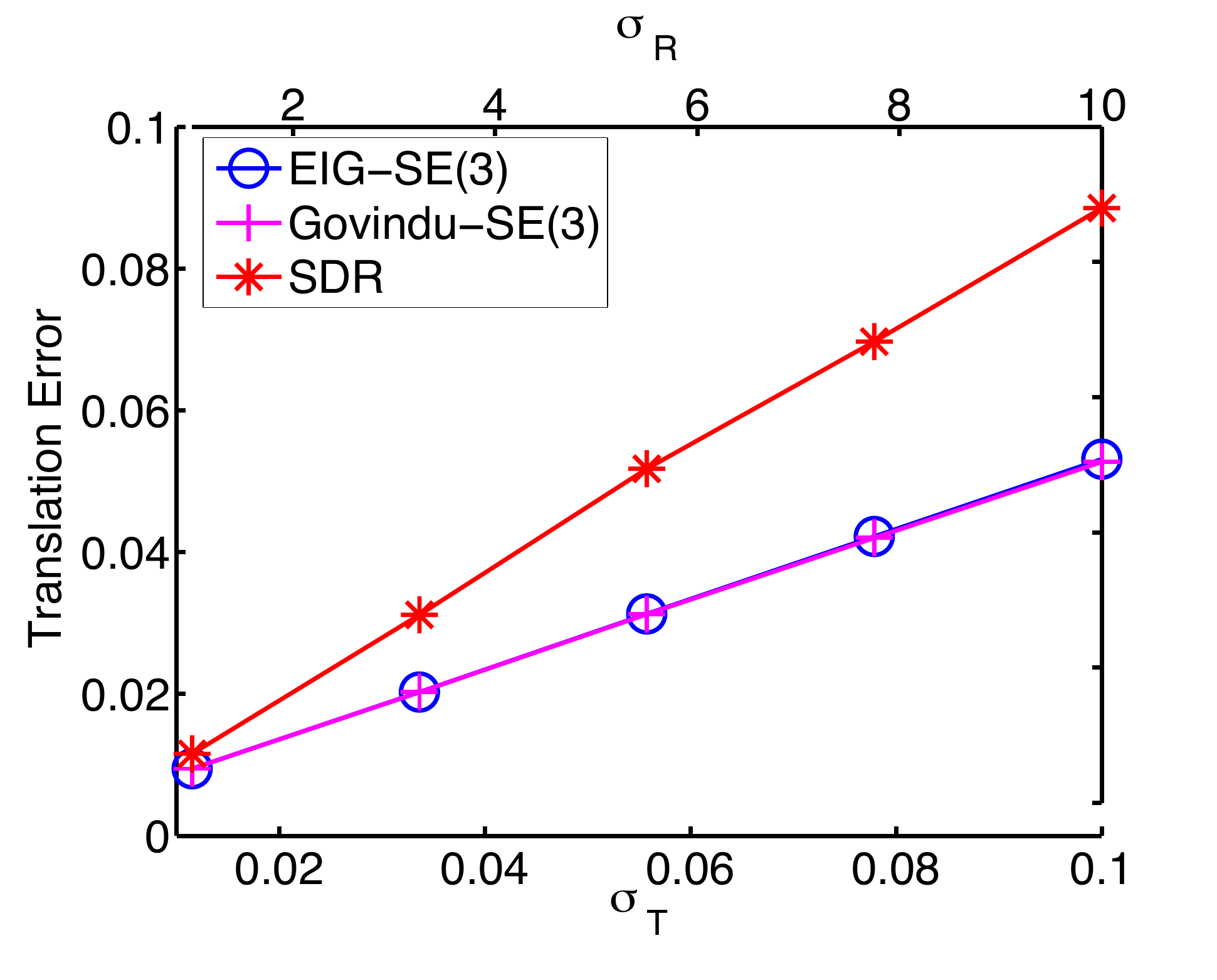} \quad
\includegraphics[width=0.33\linewidth]{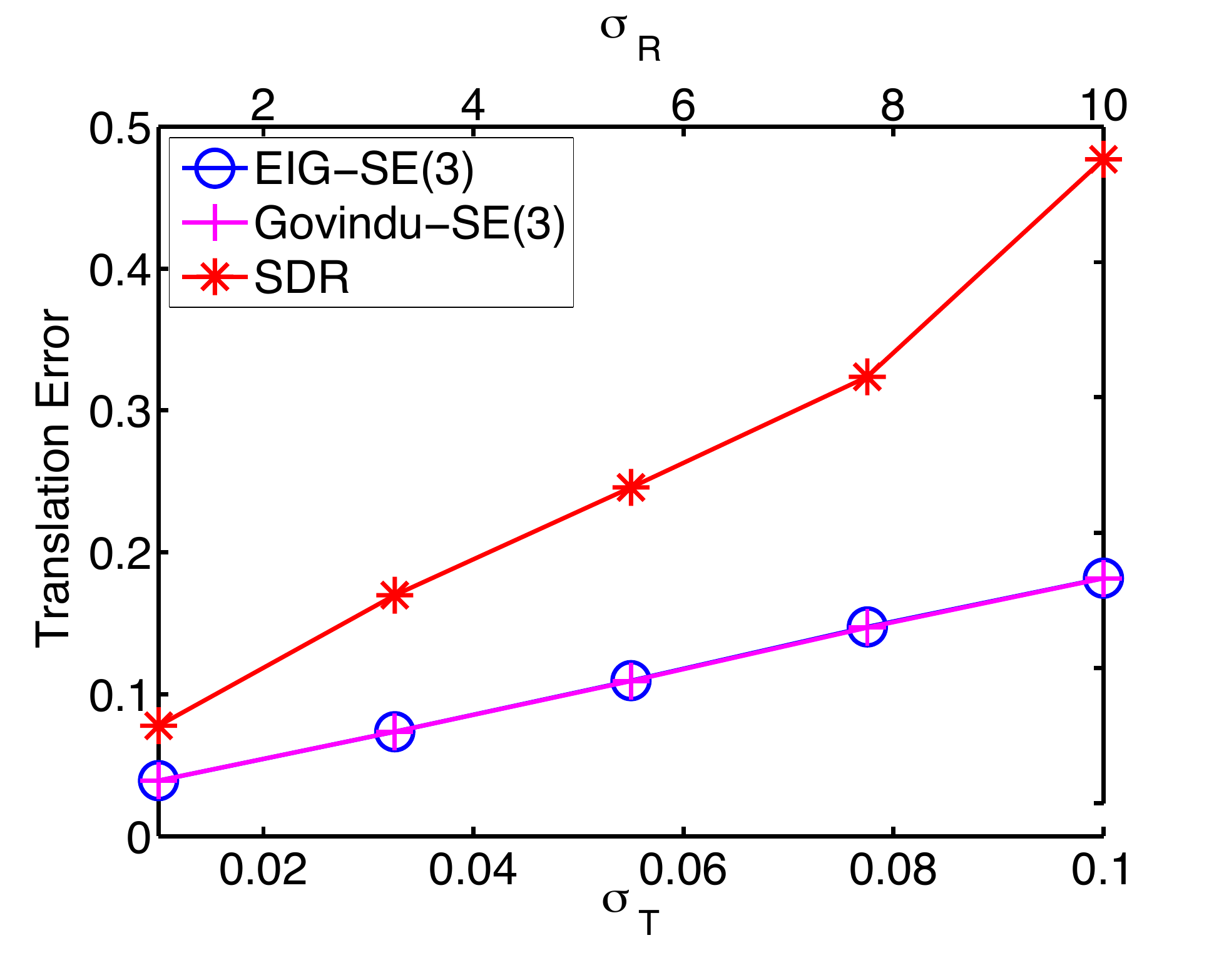}  
\caption{Mean errors on the absolute translations with noise on both rotations and translations, with $d=5$ (left) and $d=30$ (right).}
\label{exp_noiseRT}
\end{figure*}

When the measurement graph is extremely sparse ($d=5$) the methods by Govindu \cite{Govindu01} and Brand et al.~\cite{Brand04} yield larger errors than usual; by inspecting the solution it is found that this corresponds to wrong solutions concentrated around a few locations. 
This can be visualized in Fig.~\ref{exp:clustering}, which shows ground-truth and estimated positions (after alignment) for a single trial when $\sigma_T=0.08$.
Such a behaviour -- which is called  ``clustering phenomenon''   -- is analysed in \cite{Ozyesil2013stable} where the cause has been traced back to a lack of constraints on the location distances. For this reason the authors of \cite{Ozyesil2013stable} introduce ad-hoc constraints in the minimization problem, forcing the differences between locations to be ``sufficiently'' large.
On the contrary,  \namealgo and \govse , by working in $SE(3)$,  \emph{implicitly} enforce such constraints as they 
take in input the relative translations with their norm.

\begin{figure}[!htbp]
\centering
\includegraphics[width=1\linewidth]{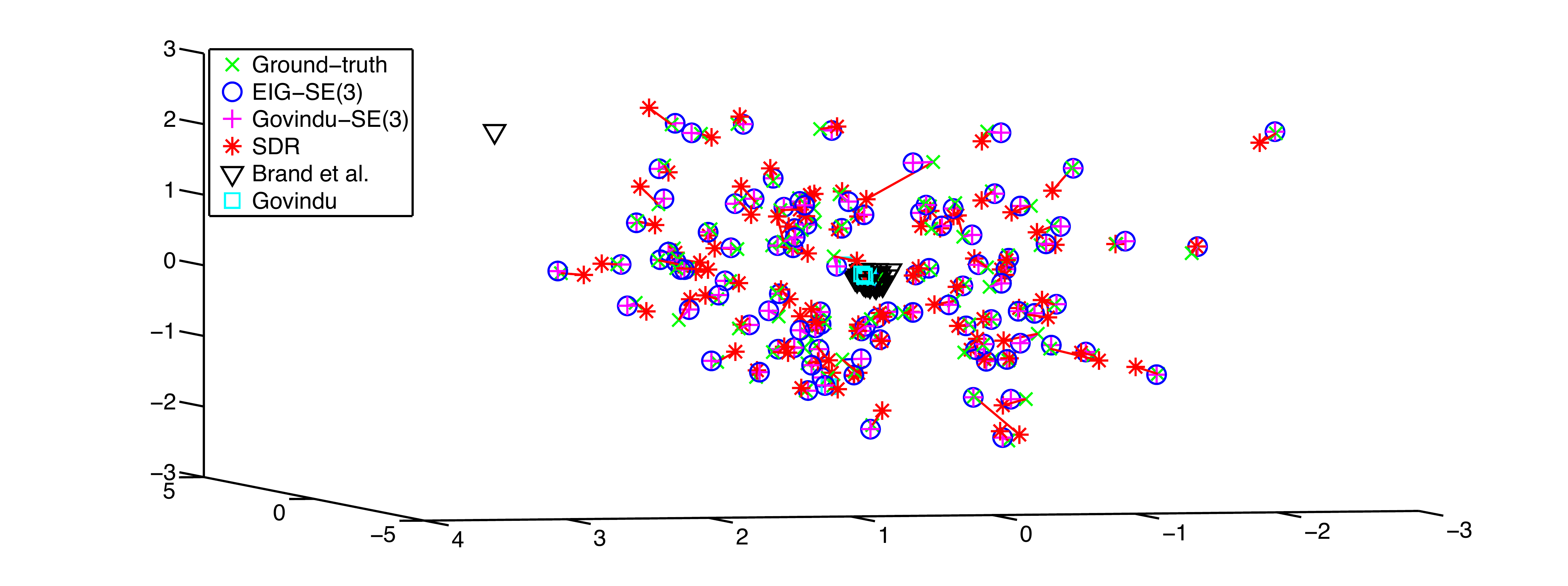} 
\caption{Clustering phenomenon. By enlarging the figure, the reader will distinguish one cluster of black triangles ({\footnotesize \color{black} $\triangledown$})  and one of cyan squares ({\footnotesize \color{cyan} $\square$}) near the origin, which correspond to the locations obtained by Brand et al.~and Govindu respectively.}
\label{exp:clustering}
\end{figure}

 \vspace{-4mm}
\paragraph{Motion synchronization.}

In this experiment we consider all the pipelines that cope with both rotation and translation synchronization, and  work in frame space, namely SDR \cite{Ozyesil2013stable} and \govse  \cite{Govindu04}.  
SDR has an original translation stage while 
rotations are computed by iterating the EIG method. Our approach and \govse  recover  both rotations and translations at the same time.

Figure \ref{exp_noiseRT} reports the mean errors on the absolute translations obtained after perturbing both relative rotations and translations. 
All the methods return good estimates, which are further improved by increasing edge connectivity, and \namealgo together with \govse  achieves the lowest errors.

We also analysed the execution time of motion synchronization, by varying the number of absolute motions from $n=100$ to $n=1000$, all the others parameters being fixed. 
More precisely, we choose the values $d=10$, $\sigma_T=0.05$ and $\sigma_R=5^{\circ}$ to define sparsity and noise.
Figure \ref{exp_time} reports the running times of the analysed algorithms as a function of the number of nodes in the measurements graph, showing that \namealgo is remarkably faster than \govse  and SDR.
Indeed SDR solves a semidefinite programming problem and \govse  uses an iterative approach in which absolute motions are updated by performing multiple averaging in the tangent space;
both these operations are more expensive than computing the least four eigenvectors of a $4n \times 4n$ matrix.

\begin{figure}[!htbp]
\centering
\includegraphics[width=0.495\linewidth]{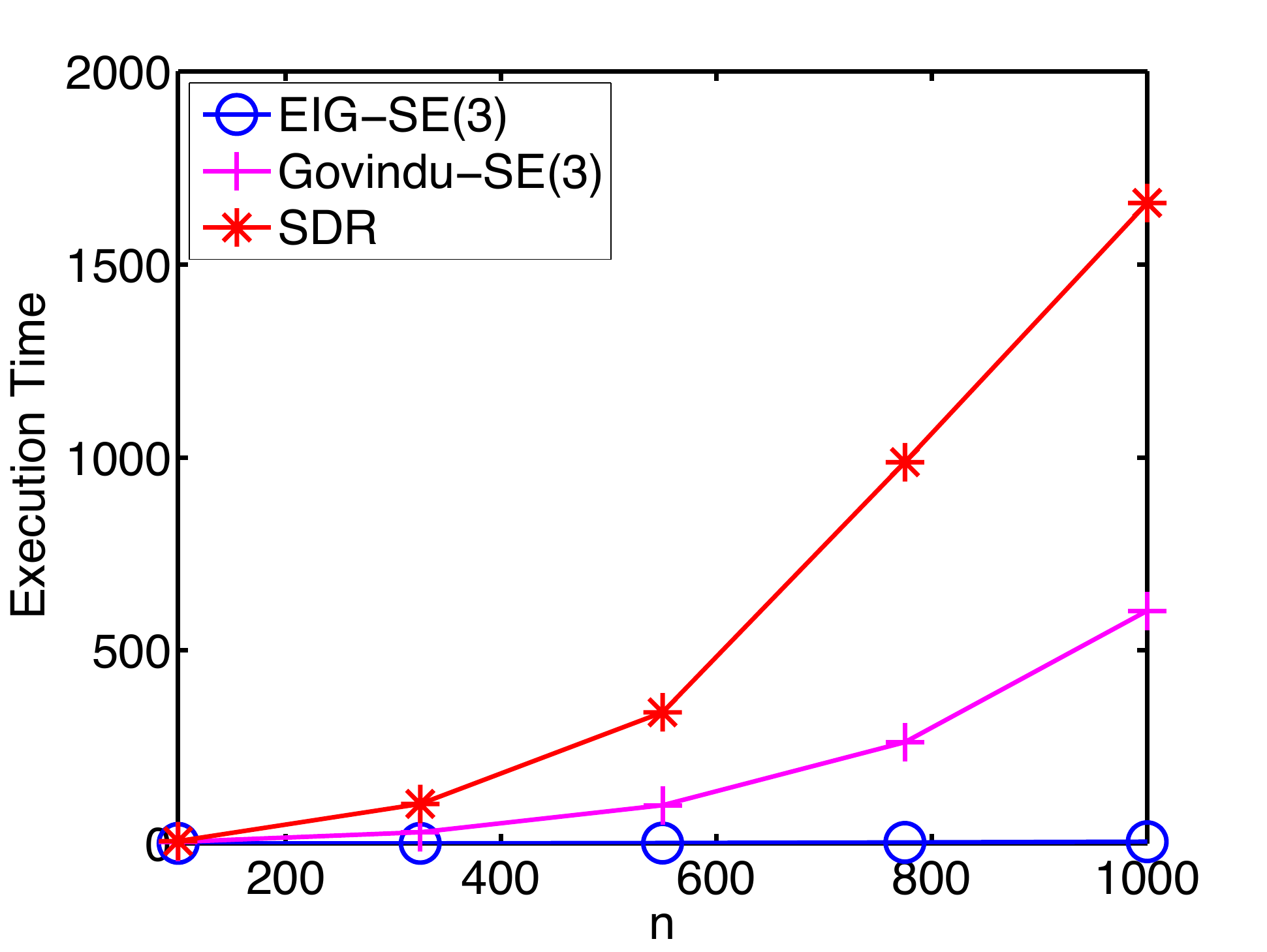} 
\includegraphics[width=0.495\linewidth]{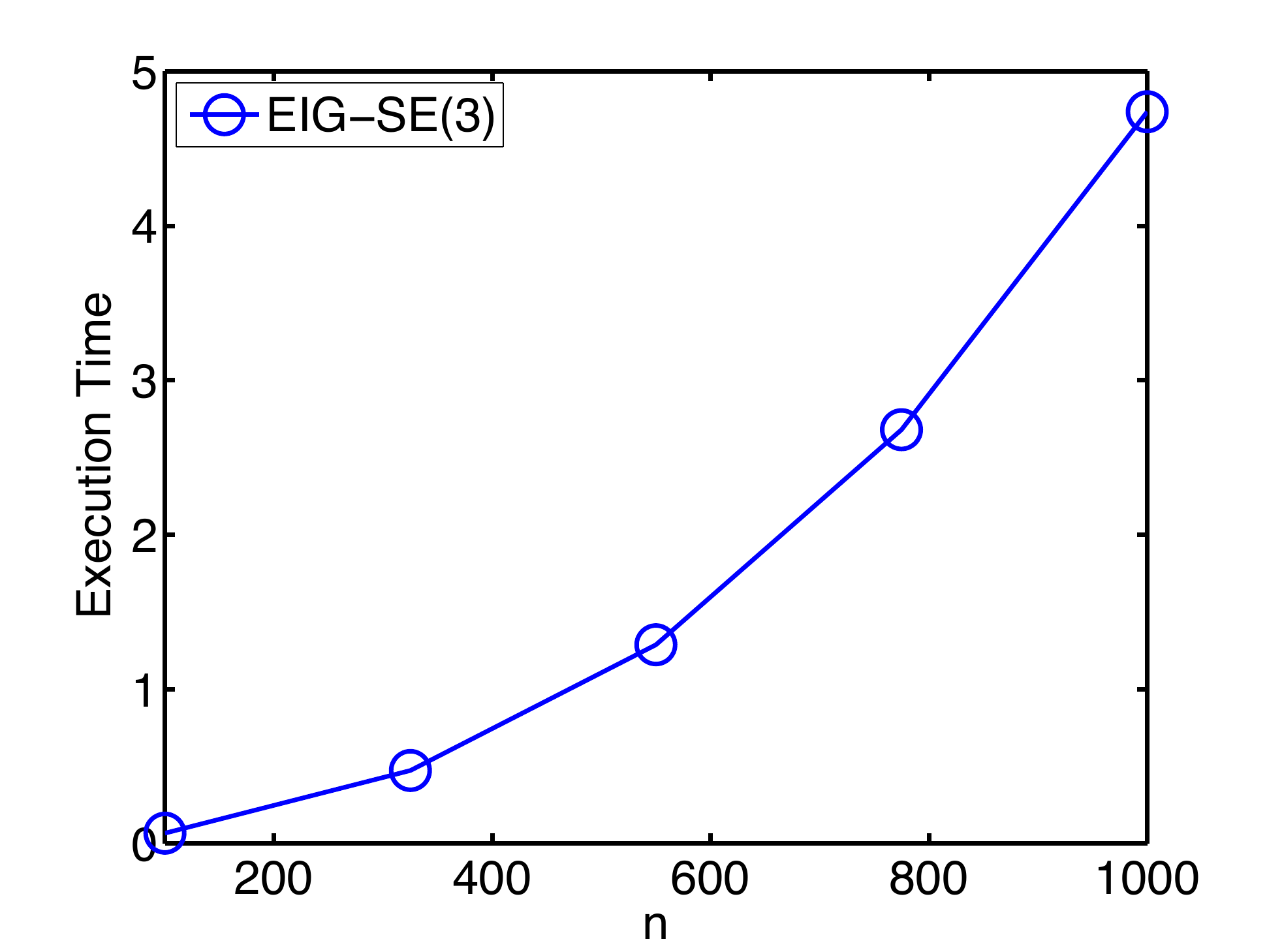}
\caption{Execution times (seconds) of motion synchronization. A magnification is shown on the right to appreciate the timing of \namealgo.}
\label{exp_time}
\end{figure}

The rundown of these experiments is that, \textbf{\namealgo achieves the same optimal accuracy of its closest competitor \cite{Govindu04} in considerably less time}.

\paragraph{Outliers influence.}

In this experiment we study the resilience to outliers of \namealgo with IRLS.
We consider $n=100$ absolute motions sampled as before and we fix $d=30$ to define sparsity.
Since we are interested in analysing exact recovery in the presence of outliers, noise is not introduced in this simulation.
The fraction of wrong relative motions -- randomly generated -- varies from $10 \%$ to $50 \%$.
Figure \ref{exp_IRLS} reports the mean errors obtained by \namealgo and its IRLS modification: the empirical breakdown point of \namealgo + IRLS is about $45\%$.

\begin{figure}[!htbp]
\centering
\includegraphics[width=0.495\linewidth]{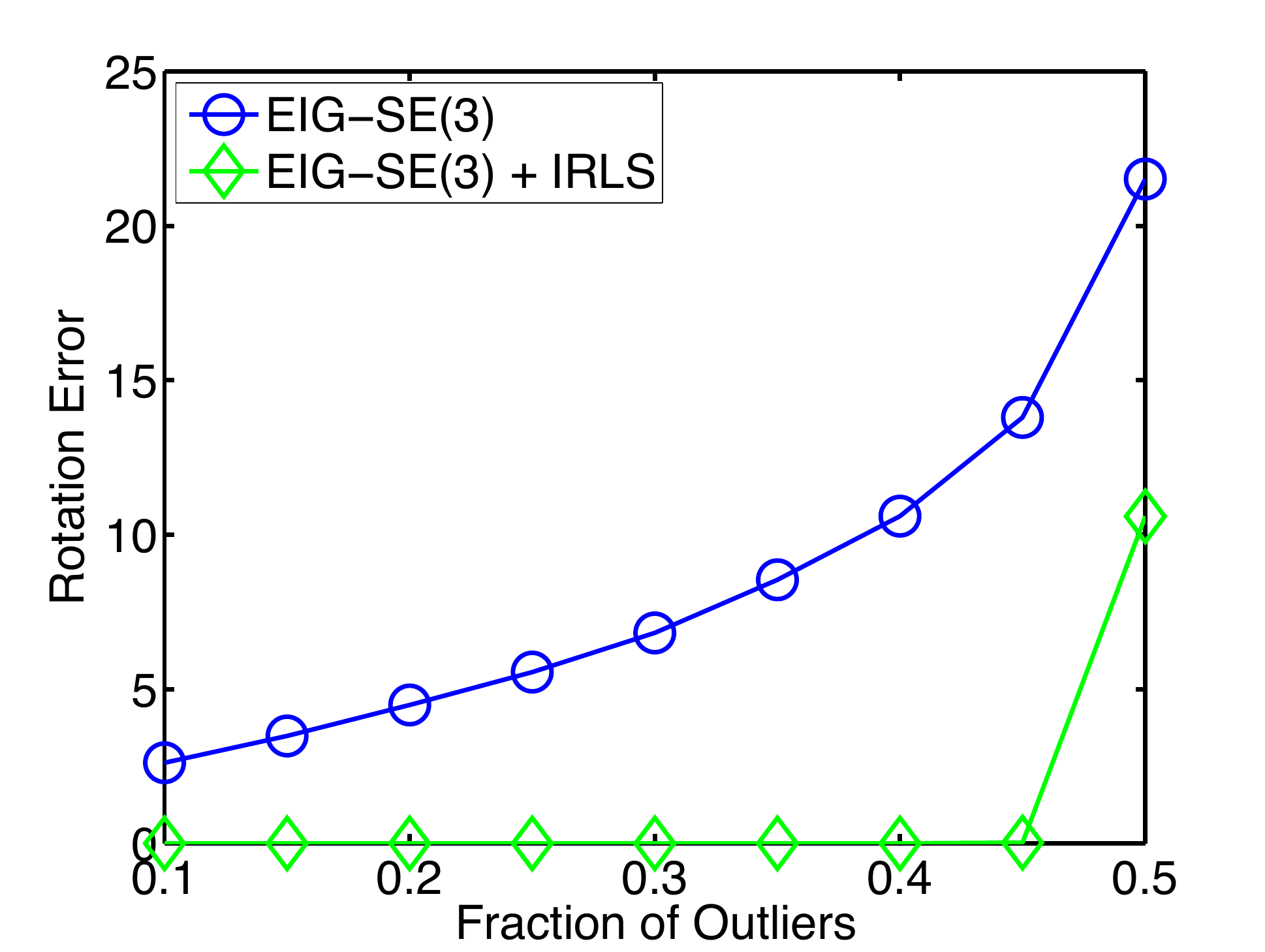} 
\includegraphics[width=0.495\linewidth]{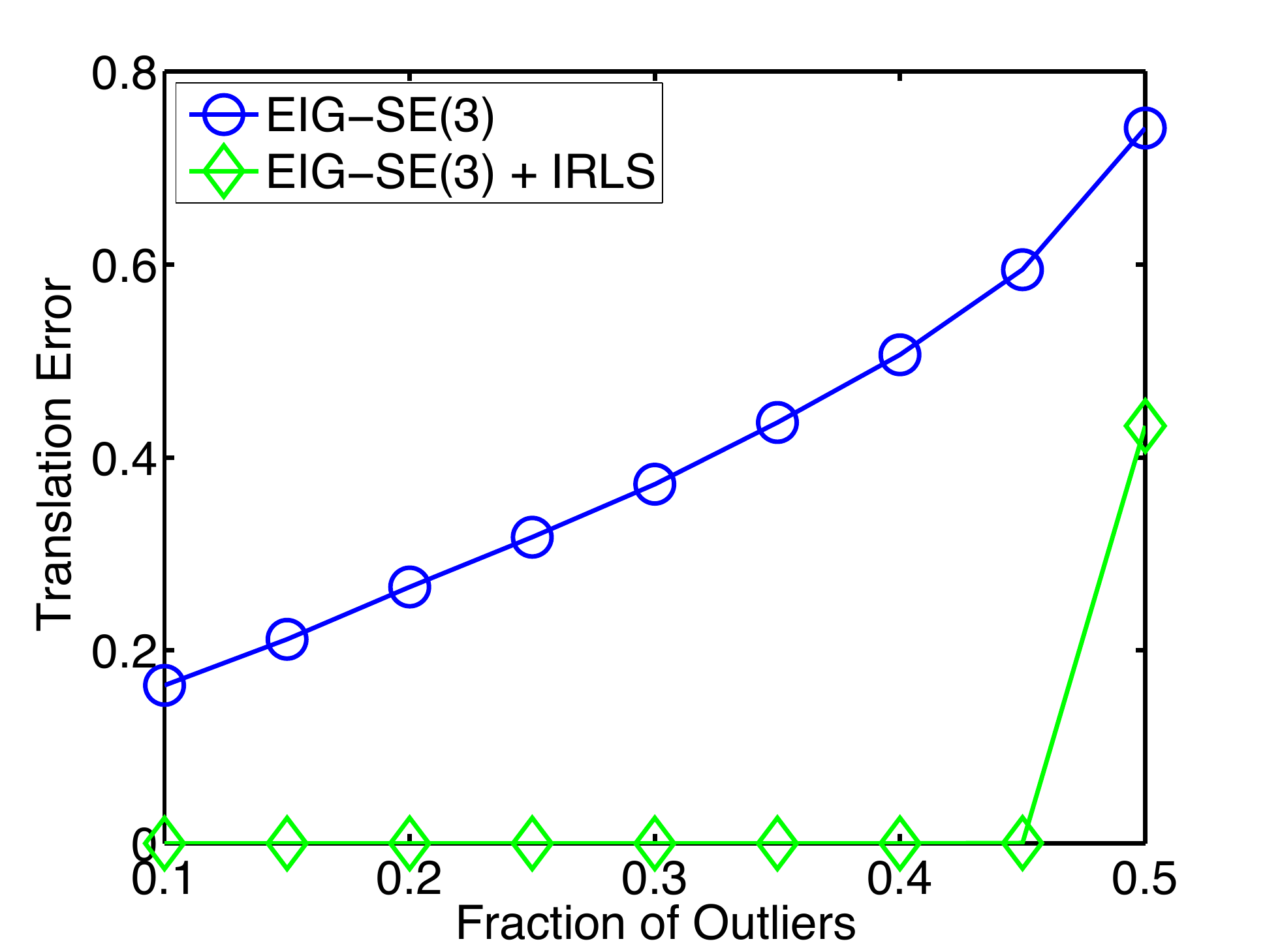}
\caption{Mean errors on the absolute motions versus outliers contamination.}
\label{exp_IRLS}
\end{figure}

\subsection{Real Data}

We apply \namealgo with IRLS to the structure from motion problem, considering both the EPFL
benchmark \cite{real_data} and unstructured, large-scale image sequences from \cite{Snavely14}. The latter are available on-line together with the relative motions, while for the EPFL
benchmark we computed them following a standard approach based on the essential matrix with a final bundle adjustment (BA) refinement of camera pairs. 


Owing to the depth-speed ambiguity, the magnitude of relative translations (also referred to as \emph{epipolar scales})  are undefined.
Therefore, the input relative motions do not fully specify elements of $SE(3)$, and the unknown scales have to be computed.


A straightforward approach (suggested in \cite{Govindu04}) consists in iteratively updating these epipolar scales, i.e. during each iteration the scale of the translation of $\widehat{M}_{ij}$ is set equal to that of $M_i M_j^{-1}$, where
$M_i$ and $M_j$ are the current estimates of camera motions.
The starting scales are all equal to 1 and the procedure is iterated until convergence. In our implementation this is combined with IRLS in the same loop: in one step we update the IRLS weights and in the next step we update the epipolar scales.

A different approach is proposed in \cite{2015arXiv150303637A}, where a two-stage method is developed for computing the epipolar scales based on the knowledge of two-view geometries only. 
First, a Minimum Cycle Basis (MCB) for the measurement graph is extracted by using Horton's Algorithm \cite{Horton}, then all the scales  are recovered simultaneously by solving a homogeneous linear system. 
This approach is based on the observation that the compatibility constraints associated to these cycles can be seen as equations in the unknown scales. In this way all the epipolar scales are computed before performing motion synchronization.

However, computing the epipolar scales is not 
part of the synchronization task, strictly speaking. As a matter of fact, this indeterminacy is an idiosyncrasy of the structure from motion problem, which is not shared, e.g., by the 
multiple point-set registration problem, where the relative motions are fully specified.
For this reason we are agnostic about the specific method for computing the scales, and 
we also  provide results obtained by using ground-truth scales, in addition to the approaches mentioned above.



\paragraph{EPFL Benchmark.}

The EPFL Benchmark datasets \cite{real_data}  contain from $8$ to $30$ images, and provide ground-truth absolute motions.

Results are reported in Tab.~\ref{tab_rotations} and Tab.~\ref{tab_translations}, which show the mean errors of motion synchronization before and after applying a two-step bundle adjustment,
as done in \cite{moulon13}, where in the first step rotations are kept fixed.
 
We consider three versions of \namealgo, which differ for the technique chosen to recover the epipolar scales, namely using ground-truth scales (GT), computing scales through \cite{2015arXiv150303637A} (MCB), and updating scales iteratively (Iter).
Our spectral solution is compared with the global SfM pipeline described by Moulon et al.~\cite{moulon13} and by  Ozyesil et al.~\cite{Ozyesil2013stable}. 
We also consider the pipeline obtained by combining the rotation synchronization technique in \cite{RGodec} with the translation synchronization method in \cite{Brand04}.
As a reference, we included in the comparison the sequential SfM pipeline \textsc{Bundler} \cite{Snavely06}.

With the exception of Moulon et al.~and \textsc{Bundler}, for which results are taken from \cite{moulon13}, all the other methods are given the same relative motions as inputs.

Both \namealgo and all the analysed techniques achieve a high precision, obtaining an average rotation error less than 0.1 degrees and an average translation error of the order of millimetres, after the final BA.  Our method is able to recover camera parameters efficiently, since the motion synchronization step takes about 1s for the largest sequences.

If we concentrate on the \namealgo-GT columns, we can see that it  achieves the optimum \emph{before} BA in most datasets,  confirming the effectiveness of our method for synchronizing relative motions, when the latter are fully specified.  Without ground-truth scales, good estimates of motion parameters are still obtained, and precision increases by using MCB rather than the iterative approach.
The error after BA is always very small and almost equal to the other methods, confirming that \namealgo  provides a good starting point for bundle adjustment.

\begin{table*}[!htb] 
\centering
\ra{1.25}
\caption{Mean angular errors (degrees) on camera rotations for the EPFL benchmark. Moulon et al.~ is missing in this table because rotation errors are not reported in \cite{moulon13}.
\label{tab_rotations}}
\smallskip
\scalebox{0.72}{
\begin{tabular}{@{}lllllllllllllll@{}}
\toprule
     & \multicolumn{2}{c}{\namealgo-GT}  
   & & \multicolumn{2}{c}{\namealgo-Iter} 
   & & \multicolumn{2}{c}{\namealgo-MCB}
   & & \multicolumn{2}{c}{Ozyesil et al.} 
   & & \multicolumn{2}{c}{\small \textsc{R-GoDec}+Brand et al.}  \\
\cmidrule{2-3} \cmidrule{5-6} \cmidrule{8-9}
\cmidrule{11-12} \cmidrule{14-15} 
Dataset 
& pre BA & post BA  & 
& pre BA & post BA  &
& pre BA & post BA  & 
& pre BA & post BA  &
& pre BA & post BA  \\
\midrule
HerzJesuP8 	& 0.04 & 0.03 && 0.03 & 0.03 && 0.03 & 0.03 && 0.06 & 0.03 && 0.04 & 0.03 \\
HerzJesuP25 & 0.06 & 0.03 && 0.06 & 0.04 && 0.06 & 0.04 && 0.14 & 0.04 && 0.13 & 0.04 \\
FountainP11 & 0.03 & 0.03 && 0.03 & 0.04 && 0.04 & 0.03 && 0.03 & 0.03 && 0.03 & 0.03 \\
EntryP10 	& 0.04 & 0.02 && 0.10 & 0.02 && 0.11 & 0.03 && 0.56 & 0.04 && 0.44 & 0.03 \\
CastleP19 	& 1.48 & 0.06 && 1.48 & 0.06 && 2.46 & 0.06 && 3.69 & 0.05 && 1.57 & 0.05 \\
CastleP30 	& 0.53 & 0.05 && 0.47 & 0.05 && 0.77 & 0.05 && 1.97 & 0.05 && 0.78 & 0.05 \\
\bottomrule
\end{tabular}
 }
\end{table*}

\begin{table*}[!htb]
\centering
\ra{1.25}
\caption{Mean errors (meters) on camera translations for the EPFL benchmark.
\label{tab_translations}}
\smallskip
\scalebox{0.72}{
\begin{tabular}{@{}lllllllllllllllllll@{}}
\toprule
     & \multicolumn{2}{c}{\namealgo-GT}  
   & & \multicolumn{2}{c}{\namealgo-Iter} 
   & & \multicolumn{2}{c}{\namealgo-MCB}
   & & \multicolumn{2}{c}{Ozyesil et al.} 
   & & \multicolumn{2}{c}{\small \textsc{R-GoDec}+Brand et al.} 
   & & \multicolumn{1}{c}{\small Moulon et al.}
   & & \multicolumn{1}{c}{\small \textsc{Bundler}} \\
\cmidrule{2-3} \cmidrule{5-6} \cmidrule{8-9}
\cmidrule{11-12} \cmidrule{14-15} \cmidrule{17-17} \cmidrule{19-19}
Dataset 
& pre BA & post BA  & 
& pre BA & post BA  &
& pre BA & post BA  & 
& pre BA & post BA  &
& pre BA & post BA  & 
&  post BA & 
&  post BA\\
\midrule
HerzJesuP8  &  0.004 & 0.004 && 0.659 & 0.004 && 0.038 & 0.004 && 0.007 & 0.005 && 0.009 & 0.004 && 0.004 && 0.016\\
HerzJesuP25 &  0.008 & 0.008 && 1.152 & 0.022 && 0.357 & 0.008 && 0.065 & 0.009 && 0.038 & 0.009 && 0.005 && 0.021\\
FountainP11 &  0.004 & 0.003 && 0.236 & 0.003 && 0.008 & 0.003 && 0.004 & 0.003 && 0.006 & 0.003 && 0.003 && 0.007 \\
EntryP10    &  0.009 & 0.008 && 0.309 & 0.008 && 0.349 & 0.009 && 0.203 & 0.010 && 0.433 & 0.009 && 0.006 && 0.055\\
CastleP19   &  0.709 & 0.034 && 4.986 & 0.034 && 3.967 & 0.035 && 1.769 & 0.032	&& 1.493 & 0.036 && 0.026 && 0.344 \\
CastleP30   &  0.212 & 0.032 && 1.974 & 0.035 && 3.866 & 0.034 && 1.393 & 0.030	&& 1.123 & 0.030 && 0.022 && 0.300 \\
\bottomrule
\end{tabular}
 }
\end{table*}



\paragraph{Large-scale Datasets.}

We test our technique on irregular large-scale collections of images taken from \cite{Snavely14}, for which recovering camera orientations/locations is challenging.

Since our {\sc Matlab} implementation of Horton's Algorithm is too slow for large datasets, we do not compute the scales through MCB in this experiment.

We compared \namealgo with a recent technique -- called 1DSfM \cite{Snavely14} -- which performs robust translation synchronization. 
Following the experiments in \cite{Snavely14}, we used the output of \textsc{Bundler} \cite{Snavely06} as reference solution, and we compute the optimal transformation between this solution and our estimate with least median of squares (LMedS), using correspondences between camera centres. 

Results are reported in Tab.~\ref{tab_large}, which shows the median errors of motion synchronization before 
applying bundle adjustment.
We also report the number of cameras reconstructed and the percentage of missing pairs, which refer to the largest parallel-rigid subgraph, extracted as explained in \cite{Ozyesil2013stable}.
The results of 1DSfM  are taken from \cite{Snavely14}, where rotation errors are not analysed. 
\namealgo with iterative scale estimate performs equal or better than 1DSfM in 7 cases out of 11, and it recovers camera rotations accurately. 

Computation times of \namealgo ({\sc Matlab} implementation on a  MacBook Air with i5 dual-core @ 1.3 GHz) reported in Tab.~\ref{tab_large} are hardly comparable with those reported in 
\cite{Snavely14}, as they refer to a compiled code on a much powerful computer. However, if we can assume that the speed gain from {\sc Matlab} 
to C++ (for non-trivial algorithms) is at least 10 times, as common wisdom suggests, we might then conjecture that \namealgo implemented in C++ would compare favourably with 1DSfM. Moreover, performing parallel computation for updating scales/weights could further improve its computational efficiency.

\begin{table}[!htbp]
\centering
\ra{1.25}
\caption{Median  errors (rotation in degrees, translation in metres) on the datasets from \cite{Snavely14} \emph{before BA}.
Boldface denotes the lowest translation error. Times are in minutes.
\label{tab_large}}
\smallskip
\scalebox{0.72}{
\begin{tabular}{@{}lccrrrrrrrrr@{}}
\toprule
    &&&&  \multicolumn{3}{c}{\namealgo-Iter}  && \multicolumn{2}{r}{1DSfM}  \\
\cmidrule{5-7} \cmidrule{9-10}
Dataset & n & miss \% && rot. & tra. &time && tra. & time\\ 
\midrule
Roman Forum	      & 1102& 89 && 2.1  &     13.5 & 16.3 && {\bf 6.1} & 3.5 \\ 
Vienna Cathedral  & 898 & 75 && 1.6  &        7 & 17.6 && {\bf 6.6} & 5   \\ 
Alamo             & 606 & 50 && 1.3  &      1.5 & 14   && {\bf 1.1} & 2.6 \\ 
Notre Dame	      & 553 & 32 && 0.8  & {\bf 0.5}& 14.7 &&       10  & 2.6 \\ 
Tower of London   & 489 & 81 && 2.8  & {\bf 7.3}& 3.0 &&       11   & 1.3 \\   
Montreal N.~Dame  & 467 & 53 && 0.6  & {\bf 0.8}& 6.3 &&       2.5  & 1.9 \\ 
Yorkminster       & 448 & 73 && 1.9  &       7.2& 3.2 && {\bf 3.4}  & 2   \\ 
Madrid Metropolis & 370 & 69 &&  5   & {\bf 9.6}& 2.5 &&       9.9  & 0.7 \\ 
NYC Library       & 358 & 71 && 3.1  & {\bf 2.5}& 2.2 && {\bf 2.5}  & 1.3 \\ 
Piazza del Popolo & 345 & 60 && 0.9  & {\bf 1.6}& 2.2 &&       3.1  & 1	  \\ 
Ellis Island	  & 240 & 33 && 0.8  & {\bf 2.8}& 1.8 &&       3.7  & 0.5 \\ 
\bottomrule
\end{tabular}
 }
\end{table}

The rundown of these experiments with real datasets shows that, endowed with IRLS to withstand outliers  and combined with a method for estimating the unknown epipolar scales, \textbf{\namealgo can compete with state-of-the-art global pipelines.}

\section{Conclusion}
We presented a new closed-form method for motion synchronization
in $SE(3)$. The method is fast and simple, being based on a spectral decomposition, 
and theoretically relevant, for it works in the manifold of rigid motions.

Our experiments showed that our method: i) has the same accuracy as its closest competitor \cite{Govindu04} but it is much faster, and ii) combined with a method for estimating the unknown translation norms, 
it can be profitably used in a global structure from motion pipeline with state of the art performances.


The {\sc Matlab} implementation of \namealgo
will be made publicly available.


{\small
\bibliographystyle{ieee}
\bibliography{references}
}

\end{document}